\documentclass{article}

\usepackage{microtype}
\usepackage{graphicx}
\usepackage{subfigure}
\usepackage{booktabs} 

\usepackage{hyperref}

\usepackage[textsize=tiny]{todonotes}
\usepackage{styles/symbols}
\usepackage{natbib}
\usepackage{geometry}

\title{Estimating Aleatoric Uncertainty in the Causal Treatment Effect}
\author{Liyuan Xu \\ \small Secondmind \\ \small \texttt{liyuan9988@gmail.com} \and Bijan Mazaheri \\ \small Dartmouth Engineering \\ \small \texttt{bijan.h.mazaheri@dartmouth.edu}}
\date{}

\begin{document}

\maketitle

\begin{abstract}

Previous work on causal inference has primarily focused on \emph{averages} and \emph{conditional averages} of treatment effects, with significantly less attention on variability and uncertainty in individual treatment responses. In this paper, we introduce the \emph{variance} of the treatment effect (VTE) and \emph{conditional variance} of treatment effect (CVTE) as the natural measure of aleatoric uncertainty inherent in treatment responses, and we demonstrate that these quantities are identifiable from observed data under mild assumptions, even in the presence of unobserved confounders. We further propose nonparametric kernel-based estimators for VTE and CVTE, and our theoretical analysis establishes their convergence. We also test the performance of our method through extensive empirical experiments on both synthetic and semi-simulated datasets, where it demonstrates superior or comparable performance to naive baselines.
\end{abstract}

\section{Introduction}
While much of statistics and machine learning focuses on accurately capturing associative relationships, causal inference \cite{pearl2009causality, peters2017elements} applies these tools to understand \emph{causal} dependencies. It seeks to understand the outcome of \emph{actions}, such as the implications of an intervention or policy. To do this, we need to address \emph{confounding bias}, which arises when confounders simultaneously affect both treatment and outcome. 

The vast majority of literature on causal inference focuses on recovering the \emph{average treatment effect (ATE)}, i.e., the \emph{average} influence of an intervention. Less attention has been paid to modeling aleatoric uncertainty in the causal effect across individuals, which is a critical measure of an intervention's risk.

Heterogeneity of treatment effects may inform critical decisions within epidemiology. For example, the emergence of a new strain or variant of a virus may lead to highly variable treatment effects with a vaccine. The ability to distinguish these events from a more homogeneous decline in effectiveness may help public health officials decide when to update a vaccine and when to administer boosters.

Classically, the heterogeneity of a causal effect is measured via the conditional average treatment effect (CATE), which represents the average responses to the intervention within a subpopulation sharing specific attributes \citep{xie2012estimating, wager2018estimation, kunzel2019metalearners, de2020two}. This, however, only captures the heterogeneity driven by the observable attributes. Consequently, CATE cannot measure the uncertainty caused by the virus strains when the affected groups are similar, i.e., the observed covariates may not fully distinguish cases that respond to the vaccine from new vaccine-resistant cases.

To fully capture this uncertainty, we propose to examine the \emph{variance} of the treatment effect (VTE), or the \emph{conditional variance} of the treatment effect (CVTE) for a more granular view of uncertainty. These quantities provide natural measures of the risk and uncertainty of a treatment response. The higher VTE suggests that treatment effects are highly individualized, while low VTEs suggest a more uniform response.

In this paper, we develop identification conditions for the VTE and CVTE, showing that they can be estimated from observable data even when a confounder exists. This extends the work by \citet{Kawakami2025Moments}, which estimates the higher-order moments of causal effects under a different setting. Motivated by this, we derive non-parametric estimators based on kernel models which provably converge to the true causal quantities. We further test the performance of our estimators empirically through experiments, based on both synthetic and semi-simulated settings, and the proposed estimator performs promisingly in all settings.

This paper is structured as follows: After reviewing the related work in \cref{sec:related-work}, we introduce the formal definition of VTE and CVTE in \cref{sec:prob-settings}. The identification condition is given in \cref{sec:identification}. We then present the kernel-based non-parametric estimator in \cref{sec:estimation}, in which we also provide the theoretical analysis. We demonstrate the empirical performance of the proposed method in \cref{sec:experiment}, covering both synthetic and semi-simulated settings.

\section{Related Works} \label{sec:related-work}

\paragraph{Heterogeneous treatment effects} Numerous approaches study heterogeneous treatment effects by looking at the causal effects within different subpopulations. For example, the average treatment effect on the treated (ATT) \cite{xie2012estimating} measures the discrepancy in causal effects between treated and untreated subpopulations, while the conditional average treatment effect (CATE) \citep{xie2012estimating, wager2018estimation, kunzel2019metalearners, de2020two} considers the effect within subpopulations sharing the same attributes. Furthermore, \citet{mazaheri2025synthetic} recovers the CATE conditioned on latent class, using ``Synthetic Potential Outcomes'' extending proximal causal inference \citep{miao2018identifying}. Although these metrics give useful insights into the heterogeneity of the causal effect, they tend to underestimate the uncertainty of the causal effect by ignoring the uncertainty caused by exogenous variables.


\paragraph{Counterfactual Distributions} 

Recently, there has been an increasing interest in exploring the causal effect on the distributional quantities, such as Gini coefficient. To address this, several methods are proposed to model the distribution of the potential outcome given a specific treatment \citep{Cernozhukov2013Inference, Kennedy2023Semi,Krikamol2021Counterfactual}. However, these approaches have a different focus from VTE, since we are interested in the variation in the \emph{gap} between potential outcomes rather than the potential outcomes themselves. This requires us to model the covariance between the potential outcomes, making these methods not directly applicable to our settings.


\paragraph{Moments of the Causal Effect}
Our work is most closely related to \citet{Kawakami2025Moments}, which investigated the (central) moments of the causal effect, since VTE is the second central moments of the causal effect. Our work, however, diverges from \citet{Kawakami2025Moments} in three key ways. Firstly, we allow for the existence of confounders, while \citet{Kawakami2025Moments} assumed that the outcome is only affected by the treatment and latent exogenous variables and tested on Randomized Controlled Trials (RCTs). Such confounding has the ability to drive different treatment responses. To handle this, we secondly introduce the conditional VTE (CVTE) to disentangle treatment uncertainty between different groups at a chosen stratification.
Finally, our kernel-based model does not require the additional monotonicity assumption employed in \citet{Kawakami2025Moments}. Our assumption of conditionally uncorrelated potential outcomes is sufficient to address our motivating examples.

\paragraph{Variance of ATE Estimators} A separate line of work seeks to quantify the variance of the ATE \emph{estimator} \citep{reifeis2022variance, matsouaka2023variance, tran2023robust} to provide the confidence interval of ATE. These approaches quantify the epistemic uncertainty in the ATE, which will decrease with additional data. In contrast, the VTE represents the aleatoric uncertainty, which is inherent in the population.

\paragraph{Kernel Mean Embedding in Causal Inference} 
In causal inference, we often need to compute the (conditional) expectation of a function of interest. The kernel mean embedding \citep{Gre2012CME,Song2013Kernel} provides a powerful tool for this without unstable density estimation, when the function is expressed with features in a reproducing kernel Hilbert space (RKHS). This approach is applied to a wide range of causal inference problems, including ATE / CATE estimation \citep{Singh2023Kernel}, front-door adjustment \citep{Singh2023Kernel}, instrumental variable regression \citep{singh2019kernel}. We also use this approach to estimate the CVTE.

\section{Problem Setting} \label{sec:prob-settings}

In this section, we introduce our novel causal parameters and show that they can be estimated from observed data. Throughout the paper, we denote a random variable in a capital letter (e.g. $A$), the realization of this random variable in lowercase (e.g. $a$), and the set where a random variable takes values in a calligraphic letter (e.g. $\mathcal{A}$). We assume that the data are generated from a distribution $P$.

\paragraph{VTE Definition} We introduce the target causal parameters using the potential outcome framework \citep{Rubin2005}. Let the treatment and the observed outcome be $A \in \{0,1\}$ and $Y \in \mathcal{Y} \subseteq [-R, R]$. We denote the potential outcome given treatment $a$ as $Y^{(a)} \in \mathcal{Y}$. Here, we assume \emph{no inference}, which means that we observe $Y = Y^{(a)}$ when $A=a$. Given these, we define \emph{the variance of treatment effect (VTE)} as the variance of the differences of potential outcomes.
\begin{definition} Define \emph{VTE} as
\begin{align*}
    \mathrm{VTE} &= \Var{Y^{(1)} - Y^{(0)}}.
\end{align*}
\end{definition}
This can capture aleatoric uncertainty in the treatment effect that other methods cannot, as discussed below.
 
\begin{figure}
    \centering
    \includegraphics[width=\linewidth]{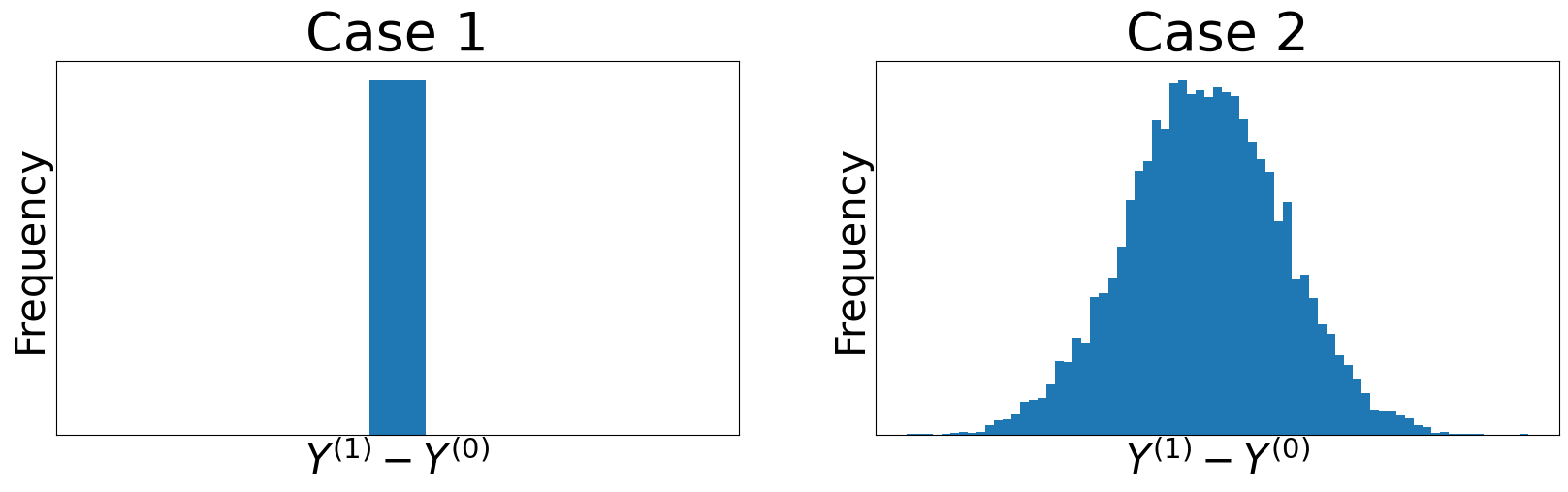}
    \caption{Histogram of treatment effect $Y^{(1)} - Y^{(0)}$}
    \label{fig:heterogeneity}
\end{figure}


\paragraph{Motivating Example} To illustrate the descriptive value of the VTE, let's consider the following two scenarios:
\begin{align*}
    \text{Case 1:}& \quad Y^{(0)} \sim \mathcal{N}(0, 1), \quad Y^{(1)} = Y^{(0)}\\
    \text{Case 2:}& \quad Y^{(0)}, Y^{(1)} \sim \mathcal{N}(0, 1), \quad Y^{(1)} \indepe  Y^{(0)}
\end{align*}
For now, let us assume \emph{no confounding} for both cases, i.e. treatment $A$ is assigned independently from potential outcomes $Y^{(0)}, Y^{(1)}$.

Although the average treatment effects (ATE) $\expect{Y^{(1)} -Y^{(0)}}$ are zero in both scenarios, we observe the larger heterogeneity in the individual causal effect in the second scenario (See \cref{fig:heterogeneity}). In particular, the treatment effect for Case 1 is homogeneous and completely certain. Meanwhile, the treatment effect for Case 2 is highly uncertain. Such information is critical when assessing the risk associated with policy decisions and medical interventions.

The average treatment effect on the treated (ATT) $\expect{Y^{(1)} -Y^{(0)} | A=1}$, is often used to model the differences in the treatment effects between treated and control groups \citep{xie2012estimating}. However, in our setting, the ATT is an inappropriate measure of uncertainty, since both cases exhibit zero ATT under the no confounding assumption.

Another relevant metric is the \emph{counterfactual distribution} \citep{Cernozhukov2013Inference}, which characterizes the distribution of the potential outcomes by
\begin{align*}
    \prob[\braket{a,a'}]{Y} = \prob{Y^{(a)} \middle| A = a'}
\end{align*}
for $a,a' \in \{0,1\}$. Again, this is not appropriate risk measurement, since all counterfactual distributions $ \prob[\braket{a,a'}]{Y}$ are the standard Gaussian distribution in both cases under no confounding. Hence, while counterfactual distributions may quantify the uncertainty of a single potential outcome, they are insufficient to quantify treatment effect uncertainty driven by the joint distribution of two potential outcomes.

While these two toy cases are identical in all of these previously mentioned metrics, the VTE is able to capture their difference:
\begin{align*}
   \text{Case 1:}& \quad  \mathrm{VTE} = \Var{Y^{(0)} - Y^{(0)}} = 0\\
    \text{Case 2:}& \quad  \mathrm{VTE} = \Var{Y^{(0)}} + \Var{Y^{(1)}} = 2.
\end{align*}
This shows the importance of VTE as a metric for heterogeneity and uncertainty.

\paragraph{Extension to Conditional VTE} It is often the case that we are interested in the variance of the treatment effect in a specific sub-population sharing a common trait. For example, we might want to determine whether the heterogeneity in vaccine effectiveness varies by geographic location, which is likely when this heterogeneity is driven by a new variant. In such cases, we may consider \emph{conditional variance of treatment effect (CVTE)} as

\begin{definition}\label{def:cvte} 
Define \emph{CVTE} as
\begin{align*}
    \mathrm{CVTE}(v) &= \Var{Y^{(1)} - Y^{(0)} \middle| V=v}
\end{align*}
where $V \in \mathcal{V}$ is some covariate observed.
\end{definition}
Such a metric is capable of identifying the drivers of heterogeneity in treatment effects. For example, if treatment effect uncertainty is fully explained by $V=v$, then $\mathrm{VTE}(v) = 0$. We can also similarly define the counterfactual quantity, within the sub-population that actually received treatment (i.e., $V=A$). For the sake of simplicity, we focus on VTE and CVTE in this work.

\section{Identification} \label{sec:identification}

The definitions of VTE and CVTE presented in \cref{sec:prob-settings} depend on potential outcomes, which are not observable in reality. In this section, we show a sufficient condition for estimating VTE and CVTE from observational data.

The majority of the literature in causal inference \citep{wager2018estimation,xie2012estimating,Shi2019,chernozhukov2022riesznet,Athey2019Causal,singh2019kernel} assumes that we observe some covariate $X$ that satisfies a condition known as \emph{ignorability}  or \emph{conditional exchangability} \citep{Rubin2005}.
\begin{assum}[Ignorability] \label{assum:ignorability}
    \begin{align*}
        Y^{(0)}, Y^{(1)} \indepe A | X 
    \end{align*}
\end{assum}
\cref{assum:ignorability} ensures that groups with different values of $A$ can be treated as ``similar'' for the purpose of computing counterfactuals, making the ATE computable as differences between the expected outcomes in both ``treated'' ($A=1$) and ``control'' ($A=0$) groups \citep{Rubin2005,rosenbaum1983central}. 

This assumption alone, however, is not sufficient to identify VTE as shown in the following theorem. To see this, let us consider the following variance composition of the VTE.
\begin{align}
    \mathrm{VTE} &= \Var{Y^{(1)}-Y^{(0)}} \nonumber\\
    &= \Var{Y^{(1)}} + \Var{Y^{(0)}} - 2\mathrm{Cov}[Y^{(1)}, Y^{(0)}] \label{eq:var-decompose-vte}
\end{align}
Although we can estimate the variance of potential outcomes $\Var{Y^{(1)}}, \Var{Y^{(0)}}$ from \cref{assum:ignorability} alone, we cannot identify the covariance $\mathrm{Cov}[Y^{(1)}, Y^{(0)}]$ since we cannot observe both potential outcomes at the same time. Indeed, we can explicitly construct the non-identifiable data process.

\begin{thm}\label{thm:nonidentifiability}
    There exists two data generation processes $P(X,A,Y^{(0)}, Y^{(1)})$ satisfying \cref{assum:ignorability} that have different VTE but the same observational distribution $P(X,A,Y)$.
\end{thm}

These data generation processes can be found in \cref{sec:proof-non-identifiability}. Therefore, we make the following additional assumption for the covariate $X$.
\begin{assum} \label{assum:discomposional}
\begin{align*}
    \expect{Y^{(0)} Y^{(1)}\middle| X} = \expect{Y^{(0)}\middle| X} \expect{Y^{(1)}\middle|X}
\end{align*}
\end{assum}

\cref{assum:discomposional} suggests that potential outcomes are \emph{uncorrelated} conditioned on $X$, meaning that all the dependence between potential outcomes is explained by the covariates. \cref{assum:discomposional} is a natural assumption whenever the covariates $X$ are sufficiently informative. For example, we would not expect the variability in vaccine effectiveness to be associated with variability in unvaccinated outcomes, so long as we observe all relevant risk factors in $X$.


Note that \cref{assum:discomposional} is automatically satisfied in the more restrictive case of independent (when conditioned on $X$) potential outcomes, but we only require \emph{uncorrelated} potential outcomes in general. 

Under \cref{assum:discomposional}, we can estimate VTE as follows.
\begin{thm} \label{thm:vte-identification}
    Let covariate $X$ satisfy both \cref{assum:ignorability,assum:discomposional}. Then, the the VTE can be estimated as 
    \begin{align*}
        \mathrm{VTE} &= \expect[X]{g_1(X) + g_0(X) -2f_1(X)f_0(X)}\\
        &\quad\quad\quad - \expect[X]{f_1(X) - f_0(X)}^2,
    \end{align*}
    where
    \begin{align*}
        f_a(X) &= \expect{Y | A=a, X}, \\
        g_a(X) &= \expect{Y^2| A=a, X}
    \end{align*}
    for $a \in \{0,1\}$.
\end{thm}
Similarly, we can also identify the CVTE as follows.
\begin{thm} \label{thm:cvte-identification}
    Let covariate $X$ satisfies both \cref{assum:ignorability,assum:discomposional}. Then, for $V \indepe Y^{(1)}, Y^{(0)} | X$ the CVTE can be estimated as 
    \begin{align*}
        &\mathrm{CVTE}(v) \\
        &\quad = \expect[X]{g_1(X) + g_0(X) -2f_1(X)f_0(X)|V=v}\\
        &\quad\quad\quad - \expect[X]{f_1(X) - f_0(X)|V=v}^2,
    \end{align*}
    where $f_a, g_a$ are defined as \cref{thm:vte-identification}.
\end{thm}

The proofs are presented in \cref{sec:proof-identifiability}. The major difference from VTE is that  we now take the conditional expectation $\expect[X|V=v]{\cdot}$, rather than the marginal expectation $\expect[X]{\cdot}$.

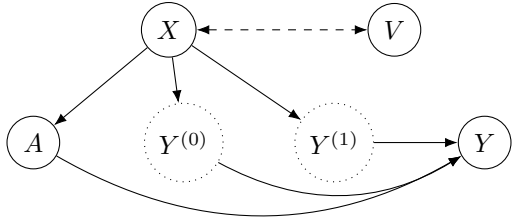
\begin{figure}
    \centering
    \begin{tikzpicture}
    \node[state] (D) at (0,0) {$A$};
    \node[state] (V) at (4.8, 1.5) {$V$};
    \node[dotstate] (Y0) at (2.0,0) {$Y^{(0)}$};
    \node[dotstate] (Y1) at (4.0,0) {$Y^{(1)}$};
    \node[state] (Y) at (6.0,0) {$Y$};
    \node[state] (U) at (1.8,1.5) {$X$};
    \path (U) edge (Y1);
    \path (U) edge (Y0);
    \path (U) edge (D);
    \path[bidirected] (U) edge (V);
    \path (Y1) edge (Y);
    \path[bend right] (Y0) edge (Y);
    \path[bend right] (D) edge (Y);
\end{tikzpicture}
    \caption{A causal graph satisfying the constraints of our setting. The dotted circles indicate that the potential outcomes are not observed.}
    \label{fig:causal-graph}
\end{figure}

The causal relationship is summarized in \cref{fig:causal-graph}, in which the bi-directional arrow between $X$ and $V$ means that we allow both directions or even a common ancestor variable. Note that \cref{assum:ignorability} prohibits having an arrow between treatment $A$ and potential outcomes $Y^{(0)}, Y^{(1)}$, and \cref{assum:discomposional} is satisfied when there is no arrow between potential outcomes $(Y^{(0)}, Y^{(1)})$, i.e. $Y^{(0)} \indepe Y^{(1)} | X$ \footnote{Technically, Assumption~\ref{assum:discomposional} allows arrows and dependence between the potential outcomes so long as the correlation, i.e., linear dependence, remains zero.}. One important example is where $V$ is a subset of $X$, which represents the variance of causal effect within a subpopulation that shares the same traits.

\begin{remark} \label{remark}
    We can rewrite the VTE in \cref{thm:vte-identification} as follows
    \begin{align*}
        \mathrm{VTE} &= \underbrace{\Var[X]{\expect{Y|A=1,X} - \expect{Y|A=0,X}}}_{\text{Variance in CATE}}\\
        &\quad\quad + \underbrace{\expect[X]{\Var{Y|A=1, X} + \Var{Y|A=0, X}}}_{\text{Exogenous uncertainty}}
    \end{align*}
    This can be interpreted as the variance decomposition in the treatment effect. The first term represents the variance in the conditional average treatment effect across different conditionings, i.e., variation due to the covariate $X$. The second term can be interpreted as the variance within the potential outcomes, which is uncertainty that is latent to our study.
    
    Note that if we estimate the CATE for each $X$, then the variance of these CATEs across different $X$ can be used to estimate the first term. This captures \emph{endogenous} variation in the treatment effect, e.g., variation in the treatment effect that is explained by observed variables. The full measurement of the VTE, however, requires the computation of \emph{exogenous} uncertainty appearing in the second term.
\end{remark}

\section{Estimation} \label{sec:estimation}

In this section, we propose empirical estimation methods for VTE and CVTE based on kernel methods. We first review the kernel ridge regression, and then present the closed-form expression of estimators.

\subsection{Preliminary}

\paragraph{Reproducing Kernel Hilbert Space} We assume that $f_a(x), g_a(x)$ is in a Reproducing Kernel Hilbert Space (RKHS) denoted as $\calH_\calX$, defined by a positive semidefinite $k: \mathcal{X}\times\mathcal{X}\to\mathbb{R}$. We assume all kernel functions are characteristic \cite{SriGreFukLanetal10} and bounded. We denote canonical features as $\phi$ and the inner product and the norm in $\calH_\calX$ as $\braket[\calH_\calX]{\cdot, \cdot}$  and $\|\cdot\|_{\calH_\calX}$, respectively. We also define the tensor product as $\otimes$, which induces the Hilbert-Schmidt operator $h_1 \otimes h_2:\calH_2\to\calH_1$ for $h_1\in\calH_1, h_2\in\calH_2$, such that $(h_1 \otimes h_2)h'_2 = \braket{h_2,h'_2}h_1$ for $h'_2\in\calH_2$. 

\paragraph{Estimator of outcome models} We use kernel ridge regression (KRR) to model $f_a, g_a$ defined in \cref{thm:vte-identification}. Given the whole dataset $\{x_i, y_i, a_i\}_{i=1}^n$, let us denote the subset of untreated data $a_i = 0$ as $\{x^{(0)}_i, y^{(0)}_i\}_{i=1}^{n_0}$ and the subset of treated data $a_i = 1$ as $\{x^{(1)}_i, y^{(1)}_i\}_{i=1}^{n_1}$. Then, KRR estimators $\hat{f}_a, \hat{g}_a$ for $a \in \{0,1\}$ are given as follows.
\begin{align}
    \hat{f}_a(x) &= \vec{k}_a(x)^\top\left(K_a + n_a\lambda_{f_a}I_{n_a}\right)^{-1} \vec{y}_a \label{eq:f-krr},\\
    \hat{g}_a(x) &= \vec{k}_a(x)^\top\left(K_a + n_a\lambda_{g_a}I_{n_a}\right)^{-1} \vec{y}^2_a, \label{eq:g-krr}
\end{align}
where $I_{n_a}$ is the identity matrix size of $n_a$, and $K_a$ is the kernel matrix for untreated/treated subsets, i.e.
\begin{align*}
    K_a &= (k(x^{(a)}_i,x^{(a)}_j))_{ij} \in \mathbb{R}^{n_a\times n_a},
\end{align*}
and vectors are 
\begin{align*}
    &\vec{k}_a(x) = [k(x, x^{(a)}_1), k(x, x^{(a)}_2), \dots, k(x, x^{(a)}_{n_a})]^\top,\\
    &\vec{y}_a = [y^{(a)}_1, \dots,y^{(a)}_{n_a}]^\top, \quad \vec{y}^2_a = [(y^{(a)}_1)^2, \dots,(y^{(a)}_{n_a})^2]^\top.
\end{align*}
Note that $\hat{f}_a$ is a standard KRR learning from the subset of the untreated $(a_i=0)$ or treated $(a_i=1)$ data points, and $g_a$ is the same except using the squared outcomes as the target of the prediction. The convergence of such models are widely studied \citep{Fischer2020Sobolev, Vito2005,Smale2007LearningTE}.

\begin{prop}[{\citealp[Theorem 1]{Fischer2020Sobolev}}]   \label{prop:krr-convergence}
    If $f_a, g_a \in \calH_\calX$ for $a\in\{0,1\}$, under a certain regularity conditions, we have
    \begin{align*}
        \|\hat{f}_a(\cdot) - f_a(\cdot)\|_{\calH_\calX} \to 0,~\|\hat{g}_a(\cdot) - g_a(\cdot)\|_{\calH_\calX} \to 0
    \end{align*}
    when $\lambda_{f_a}, \lambda_{g_a}$ is set adaptively to the datasize.
\end{prop}

The complete version of \cref{prop:krr-convergence} can be found in \cref{sec:proof-vte-estimation}, in which we present the detailed regularity conditions and the rate of convergence.
We use this result to construct consistent VTE and CVTE estimators as follows.

\subsection{VTE Estimator}
Assuming we can have consistent outcome models $\hat{f}_a, \hat{g}_a$, the VTE can be estimated by taking the empirical average over the data sample.

\begin{thm}\label{thm:vte-emp-sol}
   Given dataset $\{x_i, y_i, a_i\}_{i=1}^n$, the estimator of the VTE defined as
   \begin{align*}
       \widehat{\mathrm{VTE}} &= \frac1n \sum_{i=1}^n \hat{g}_1(x_i) +  \hat{g}_0(x_i) - 2 \hat{f}_1(x_i)\hat{f}_0(x_i)\\
       &\quad - \left(\frac1n \sum_{i=1}^n \hat{f}_1(x_i) - \hat{f}_0(x_i)\right)^2
   \end{align*}
   converges to the true VTE under the conditions of \cref{thm:vte-identification} and \cref{prop:krr-convergence},
   where $\hat{f}_1, \hat{f}_0, \hat{g}_1, \hat{g}_0$ are the KRR models defined in \eqref{eq:f-krr} and \eqref{eq:g-krr}. The convergence rate is the same as the convergence rate of KRR in \cref{prop:krr-convergence}.
\end{thm}
The proof is shown in \cref{sec:proof-vte-estimation}, which combines the
convergence results of KRR \cref{prop:krr-convergence} and Hoeffding's inequality. The time complexity of $O(n^3)$, which can be reduced by the usual Cholesky or Nystr{\"o}m techniques.

\subsection{CVTE Estimator}
Unlike VTE, CVTE requires computing the conditional expectation $\expect[X|V]{\cdot}$. To do this, we use \emph{conditional mean embedding} \citep{Song2009Hilbert,Gre2012CME,Li2022optimal,Park2020CME}.

\begin{lem}\label{lem:cme-to-cvte}
    Assume $f_a, g_a \in \mathcal{H}_\calX$. Under the condition of \cref{thm:cvte-identification} 
    \begin{align*}
        \mathrm{CVTE}(v)  &= \braket[\calH_\calX]{{g}_1 + {g}_0, {\mu}_{X|V}(v)}\\
        &\quad - 2\braket[\calH_\calX]{{f}_1, {C}_{XX|V}(v){f}_0} \\
        &\quad\quad - \braket[\calH_\calX]{{f}_1-{f}_0, {\mu}_{X|V}(v)}^2
    \end{align*}
    for 
    \begin{align*}
        {\mu}_{X|V}(v) &= \expect{\phi(X)|V=v},\\
        C_{XX|V}(v) &= \expect{\phi(X) \otimes \phi(X)|V=v}
    \end{align*}
\end{lem}
The proof can be found in \cref{sec:proof-vte-estimation}, which uses the linearity of expectation and the reproducing property.
Here, ${\mu}_{X|V}(v), C_{XX|V}(v)$ are the \emph{conditional mean embedding} \citep{Song2009Hilbert,Gre2012CME,Li2022optimal,Park2020CME}. Under the regularity condition, it is shown that this can be estimated as a solution of the least-squares problem.

\begin{prop}[{\citealp[Thoerem 1]{singh2019kernel}}]  \label{prop:cme}
    Under some regularity conditions, given data $\{x_i, v_i\}_{i=1}^n$, we have 
    \begin{align*}
        &\left\| \mu_{X|V}(v) -  \sum_{i=1}^n w_i(v) \phi(x_i)\right\|_{\mathcal{H}_\calX} \to 0,\\
        &\left\| C_{XX|V}(v) -  \sum_{i=1}^n w_i(v) \phi(x_i)\otimes \phi(x_i) \right\|_{\mathcal{H}_\calX} \to 0,
    \end{align*}
    for all $v \in \mathcal{V}$, where $w_i(v)$ is the $i$-th element of the vector
    \begin{align*}
        \vec{w}(v) = \left(K_V + n\lambda_V I\right)^{-1}\vec{k}_V(v).
    \end{align*}
    Here, $K_V, \vec{k}_V$ are defined as
    \begin{align*}
        &K_V = (k_V(v_i, v_j))_{ij} \in \mathbb{R}^{n\times n},\\
        &\vec{k}_V(v) = [k_V(v_1, v), k_V(v_2, v), \dots, k_V(v_n, v)]^\top \in \mathbb{R}^{n},
    \end{align*}
    for kernel function $k_V:\mathcal{V} \times \mathcal{V} \to \mathbb{R}$.
\end{prop}

The complete version of \cref{prop:cme} can be found in \cref{sec:proof-vte-estimation}. Using this estimator of conditional mean embeddings, we can construct the empirical estimator of CVTE as follows.

\begin{thm} \label{thm:cvte-emp-sol}
   Given dataset $\{x_i, y_i, a_i, v_i\}_{i=1}^n$, the following estimator, $\widehat{\mathrm{CVTE}}$, converges to the true CVTE under the condition of \cref{thm:vte-identification,prop:krr-convergence,prop:cme}.
\begin{align*}
     &\widehat{\mathrm{CVTE}}(v) = \\
     &\quad\sum_{i=1}^n w_i(v)\left(\hat{g}_1(x_i) +  \hat{g}_0(x_i) - 2 \hat{f}_1(x_i)\hat{f}_0(x_i)\right)\\
       &\qquad - \left(\sum_{i=1}^n  w_i(v) \left(\hat{f}_1(x_i) - \hat{f}_0(x_i)\right)\right)^2
\end{align*}
for the KRR models $\hat{f}_1, \hat{f}_0, \hat{g}_1, \hat{g}_0$ defined in \eqref{eq:f-krr}, \eqref{eq:g-krr}, respectively, and $w_i(v)$ are the weights for conditional mean embedding defined in \cref{prop:cme}. The convergence rate is equal to the slower rate of KRR estimation in \cref{prop:krr-convergence} or the conditional mean embedding estimation in \cref{prop:cme}.
\end{thm}

The proof can be found in \cref{sec:proof-vte-estimation}. The only difference between the CVTE estimator and the VTE estimator in \cref{thm:vte-emp-sol} is that the VTE estimator averages over samples with equal weights, while CVTE considers the \emph{weighted} average induced from the conditional mean embedding $w_i(v)$. Therefore, compared to the VTE estimator, CVTE requires $O(n^3)$ additional operations to obtain weights $w_i(v)$, which can be accelerated again by the usual Cholesky or Nystr{\"o}m techniques.

Note that solutions in \cref{thm:cvte-emp-sol} cannot be used for the case where $V \subseteq X$, since it does not satisfy the assumptions in \cref{prop:cme}. We will present an alternative solution for such cases in \cref{sec:emp-cvte-identity}.

\section{Experiments} \label{sec:experiment}

In this section, we report the empirical performances of our kernel estimators. We first discuss the baseline methods we propose, and then, we present the results for both VTE and CVTE estimations.

\subsection{Baselines}
Since VTE or CVTE has not been discussed before, we made sensible adjustments to devise following baselines. We first describe the baseline for VTE, and adaptation to CVTE will be discussed subsequently.

\paragraph{Naive Estimator} We introduce a naive baseline that ignores the confounding $P(Y^{(a)}) = P(Y|A=a)$ and assumes potential outcomes are marginally uncorrelated $\mathrm{Cov}(Y^{(1)}, Y^{(0)}) = 0$. Under these simplifications, we can see that VTE has the simple form
\begin{align*}
    \mathrm{VTE}  &= \Var{Y|A=1} + \Var{Y|A=0}, 
\end{align*}
which can be estimated from data.

\paragraph{CATE-based Estimator} Many models exist to estimate the CATE $\tau(x)$ defined as 
\begin{align*}
    \tau(x) = \expect{Y|A=1, X=x} - \expect{Y|A=0, X=x}.
\end{align*}
We might be tempted to use the variance of the estimated CATE $\Var[X]{\tau(X)}$ as the estimation of VTE. This is, however, a biased estimator, as discussed in Remark~\ref{remark}. We compare to this estimator using causal forests \citep{wager2018estimation} to estimate the CATE.

\paragraph{Match-based Estimator} Causal matching \citep{Stuart2010Matching} is an approach to replicate a randomized experiment as closely as possible by matching treated and untreated units. Specifically, for each subject $i$ in the treated group, the unobserved potential outcome $Y^{(0)}$ is estimated by the average of observed outcomes of the $k$-nearest neighbor of the untreated subject. We repeat the same process to infer the potential outcome of the treated $Y^{(1)}$ for the untreated groups, and then compute the VTE $\Var{Y^{(1)} - Y^{(0)}}$ using all data. We attempted to match using Euclidean distance on the covariance matrix $X$ and propensity score matching (PSM) \citep{Rubin1996Matching}. 

\subsection{Experiment Results}

\paragraph{Synthetic setting for VTE estimation}
To test the performance of baselines, we employ the following data generation process, which is adopted from \citet{Singh2023Kernel}. The covariance $X \in \mathbb{R}^{100}$ is sampled from multivariate Gaussian $\mathcal{N}(0, \Sigma)$ where the covariance matrix $\Sigma$ is given as  
\begin{align*}
    \Sigma_{ij} = \begin{cases}
        1 & (i=j)\\
        0.5 & (|i-j| = 1)\\
        0 & (\text{otherwise})
    \end{cases}.
\end{align*}
Then, the probability of treatment $A=1$ is given as
\begin{align*}
    \prob{A=1} = \Phi(\vec{\beta}^\top X),
\end{align*}
where $\Phi$ is CDF of standard normal distribution and $\vec{\beta}$ is the vector whose $i$-th element is $\beta_i = 1/(i+1)^2$. Given treatment $X$ and outcome $A$, the outcome is generated as 
\begin{align*}
    Y = \vec{\beta}^\top X + AX_1 + \varepsilon, \quad \varepsilon \sim \mathcal{N}(0, 1)
\end{align*}
where $X_1$ denotes the first element of the covariate $X$. \cref{tab:vte_synthetic} summarizes the mean absolute error for 20 repetitions on various data sizes, and \cref{fig:vte_synthetic} presents the predictive distribution.
We used a Gaussian kernel for the proposed estimator, with the bandwidth set to the median of pairwise distances. The regularization parameters $\lambda_{f_a}, \lambda_{g_a}$ are set by minimizing the leave-one-out error of the kernel ridge regression.

\begin{table}[]
    \centering
    \begin{tabular}{c|ccc}
        &\multicolumn{3}{c}{Data size} \\
         &500 & 1000 &5000  \\\hline
         Causal Forest &2.34 (0.17) &2.31 (0.26) & 2.13 (0.10) \\
         Naive &3.77 (0.47) &3.79 (0.37)& 3.87 (0.14) \\
         Euc. Match. &0.82 (0.28) & 0.69	(0.29)& 0.50 (0.11) \\
         PSM &0.43 (0.40) & 0.41 (0.35)& 0.18 (0.11)\\\hline
         Proposed &0.35 (0.33) & 0.38 (0.24) & 0.15 (0.11)\\
    \end{tabular}
    \caption{Mean and standard error of absolute error for synthetic VTE estimation. }
    \label{tab:vte_synthetic}
\end{table}

\begin{figure}
    \centering
    \includegraphics[width=\linewidth]{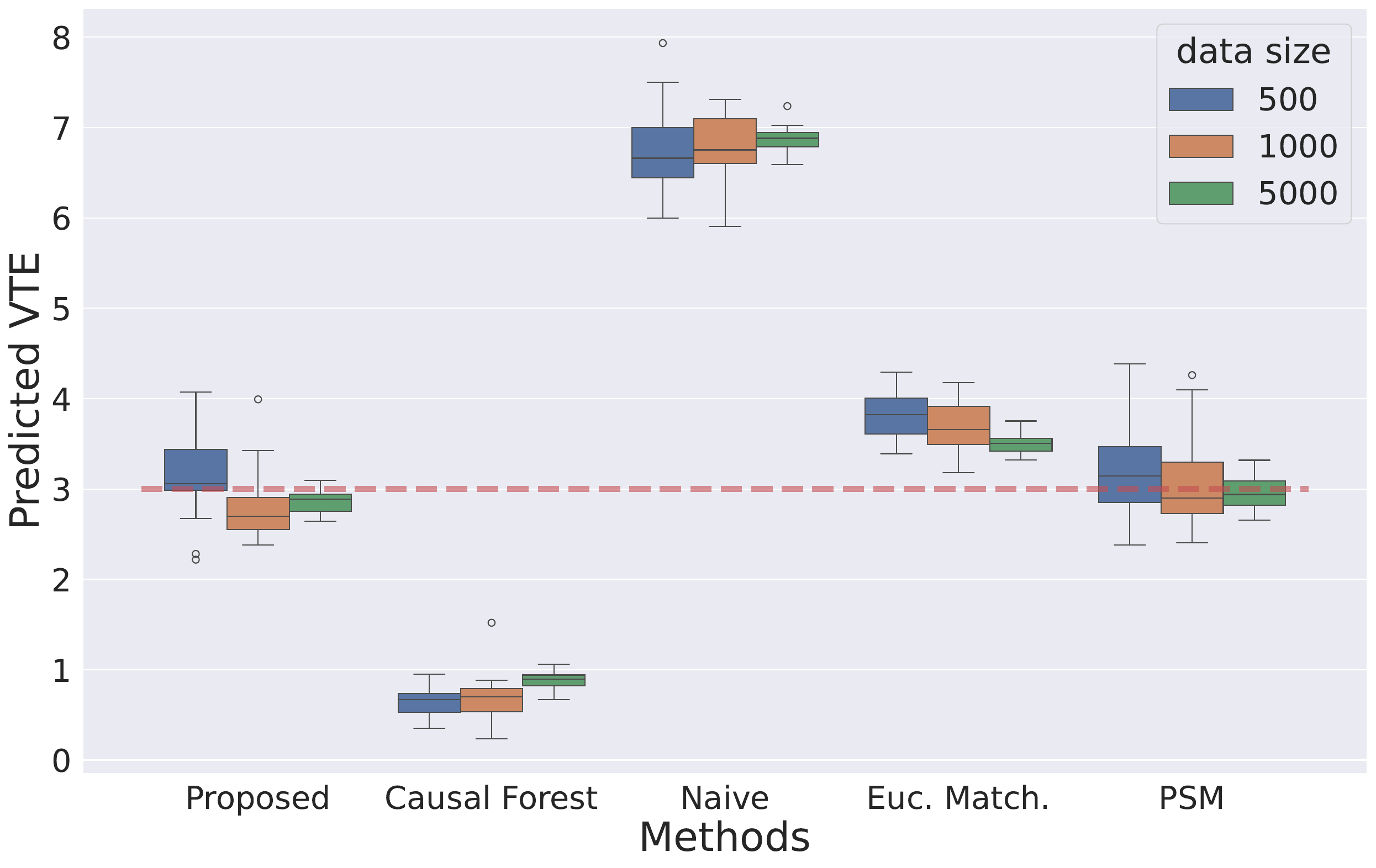}
    \caption{Box plot for VTE prediction. The red line shows the true VTE.}
    \label{fig:vte_synthetic}
\end{figure}

From \cref{fig:vte_synthetic}, we can see that Causal Forest and Naive baseline exhibit the large bias discussed in \cref{sec:prob-settings}. Although Euc. Match. also consistently overestimates the VTE, PSM seems to have little minimum bias, which is because the propensity score model is almost correctly satisfied. Still, the proposed model consistently outperforms all baselines in all data sizes, as shown in \cref{tab:vte_synthetic}.

\paragraph{Synthetic Setting for CVTE}

We use the same data generation process when testing CVTE, instead aiming to estimate the CVTE with $X_2 = 0$, where $X_2$ denotes the second element of the covariate $X$. Since $V$ is a subset of $X$, we used the alternative solution presented in \cref{sec:emp-cvte-identity}. To compute the conditioning, we only used the data subset of $|X_2|\leq 0.1$ in the matching and naive baseline. For the CATE-based Estimator, we used all data to train CATE $\tau(x)$, and computed the variance of $\tau(x)$ within the data subset of $|X_2|\leq 0.1$. We determined the remaining hyperparameters using the same procedure as for synthetic VTE, and the mean absolute error and the predictive error over 20 repeats for various data sizes are shown in \cref{tab:cvte_synthetic} and \cref{fig:cvte_synthetic}, respectively.

\begin{table}[]
    \centering
    \begin{tabular}{c|ccc}
        &\multicolumn{3}{c}{Data size} \\
         &500 & 1000 &5000  \\\hline
         Causal Forest &1.98 (0.15) & 1.94 (0.18) & 1.83 (0.08) \\
         Naive &2.21 (1.25)& 2.84 (0.75)& 2.95 (0.40) \\
         Euc. Match. &1.08 (1.09)& 1.25 (0.68)& 0.93 (0.33) \\
         PSM &1.63 (1.67)& 1.06 (0.73) & 0.48	(0.57)\\\hline
         Proposed &0.61 (0.33) & 0.40 (0.16) & 0.19 (0.12)
\\
    \end{tabular}
    \caption{Mean and standard error of absolute error for synthetic CVTE estimation. }
    \label{tab:cvte_synthetic}
\end{table}

\begin{figure}
    \centering
    \includegraphics[width=\linewidth]{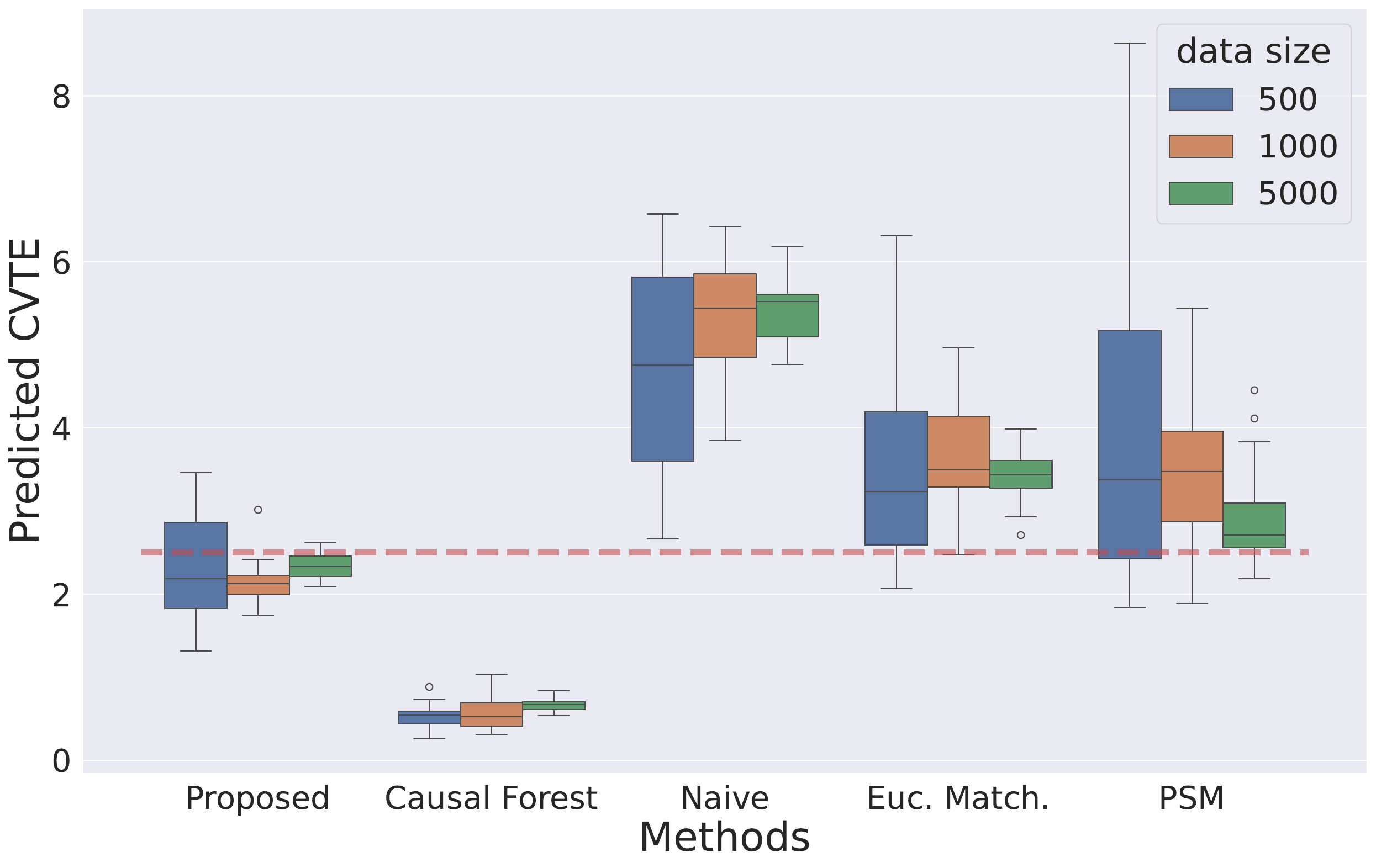}
    \caption{Box plot for CVTE prediction. The red line shows the true CVTE.}
    \label{fig:cvte_synthetic}
\end{figure}

Similar to the VTE setting, we see bias in the Causal Forest and Naive baselines and we observe the larger variance in the matching baselines. This is because we only used the small subset of data $|X_2| \leq 0.1$ to estimate the CVTE. Our proposed method, again, outperforms all baselines.

\paragraph{Semi-Simulated setting for VTE}

The IHDP dataset is widely used to evaluate the performance of the estimators for the ATE \citep{Shi2019,chernozhukov2022riesznet,Athey2019Causal}. This is a semi-synthetic data set based on the Infant Health and Development Program (IHDP) \citep{IHDP}, which examines the effects of home visits and specialized clinics on future developmental and health outcomes for premature infants. Following prior work, we generate outcomes and binary treatments from the 25-dimensional observable confounder, which comprises both continuous and categorical variables. Each dataset consists of 747 observations, and we normalized the outputs so that the observed outputs have unit variance. The performance in 1000 datasets is summarized in \cref{tab:vte_ihdp}. 

\begin{table}[]
    \centering
    \begin{tabular}{c|c}
        & \\\hline
         Causal Forest & 0.52 (0.11)\\
         Naive & 0.17 (0.08) \\
         Euc. Match & 0.06 (0.04) \\
         PSM	& 0.16 (0.09) \\ \hline
        Proposed & 0.10	(0.08)
    \end{tabular}
    \caption{Mean and standard error of absolute error for IHDP dataset. }
    \label{tab:vte_ihdp}
\end{table}

Although we see that proposed methods outperform Causal Forest and Naive baselines, which yield biased estimation, Euc. Match. achieves better performance than the proposed method (though this difference is not statistically significant). We suspect that this is because the covariate $X$ contains both continuous and categorical variables, a challenging setting for kernel ridge regression with Gaussian kernels. We could improve performance by tuning the appropriate kernel function or by using neural networks to model $f_a$ and $g_a$, though these modifications would not provide the same rigorous theoretical foundation.

\section{Conclusion}

We have proposed a novel causal quantity, the VTE, to measure the uncertainty in the causal effect. We have also extended VTE to CVTE, which measures uncertainty within subgroups sharing the same traits. We discussed the sufficient conditions for estimating these quantities from observable data, based on which we proposed kernel-based estimators. We established the consistency of the proposed method using kernel ridge regression and kernel conditional mean embedding arguments. Our empirical evaluation shows that the proposed method outperforms existing estimators, especially for CVTE estimation.

As future work, it would be promising to consider a \emph{doubly robust estimator} for VTE and CVTE. A current limitation of our consistency result is that the rate of convergence cannot exceed that of the kernel ridge regressor. A doubly robust estimator could achieve a better rate by introducing additional nuisance parameters to estimate, leveraging the \emph{mixed bias property}. Furthermore, it would be interesting to consider similar kernel estimators for higher moments of the causal effect. Prior work \citep{Kawakami2025Moments} proposed quantifying skewness and kurtosis, which depend on the third and fourth moments. It would be beneficial to develop estimators for these higher-order moments under confounding bias to provide a more complete picture of the causal effect distribution.

\section*{Impact Statement}

This paper presents work whose goal is to advance the field of machine learning. There are many potential societal consequences of our work, none of which we feel must be specifically highlighted here.

\section*{Acknowledgments}
Support for Bijan Mazaheri was provided by the Advanced Research Concepts (ARC) COMPASS program, sponsored by the Defense Advanced Research Projects Agency (DARPA) under agreement number HR001-25-3-0212.

\bibliography{reference}

@article{Cernozhukov2013Inference,
 author = {Victor Chernozhukov and Iván Fernández-Val and Blaise Melly},
 journal = {Econometrica},
 number = {6},
 pages = {2205--2268},
 title = {INFERENCE ON COUNTERFACTUAL DISTRIBUTIONS},
 volume = {81},
 year = {2013}
}

@ARTICLE{Song2013Kernel,
  author={Song, Le and Fukumizu, Kenji and Gretton, Arthur},
  journal={IEEE Signal Processing Magazine}, 
  title={Kernel Embeddings of Conditional Distributions: A Unified Kernel Framework for Nonparametric Inference in Graphical Models}, 
  year={2013},
  volume={30},
  number={4},
  pages={98-111},
}

@article{miao2018identifying,
  title={Identifying causal effects with proxy variables of an unmeasured confounder},
  author={Miao, Wang and Geng, Zhi and Tchetgen Tchetgen, Eric J},
  journal={Biometrika},
  volume={105},
  number={4},
  pages={987--993},
  year={2018},
  publisher={Oxford University Press}
}

@article{kunzel2019metalearners,
  title={Metalearners for estimating heterogeneous treatment effects using machine learning},
  author={K{\"u}nzel, S{\"o}ren R and Sekhon, Jasjeet S and Bickel, Peter J and Yu, Bin},
  journal={Proceedings of the national academy of sciences},
  volume={116},
  number={10},
  pages={4156--4165},
  year={2019},
  publisher={National Acad Sciences}
}

@article{xie2012estimating,
  title={Estimating heterogeneous treatment effects with observational data},
  author={Xie, Yu and Brand, Jennie E and Jann, Ben},
  journal={Sociological methodology},
  volume={42},
  number={1},
  pages={314--347},
  year={2012},
  publisher={Sage Publications Sage CA: Los Angeles, CA}
}

@InProceedings{mazaheri2025synthetic,
  title = 	 {Synthetic Potential Outcomes and Causal Mixture Identifiability},
  author =       {Mazaheri, Bijan and Squires, Chandler and Uhler, Caroline},
  booktitle = 	 {Proceedings of The 28th International Conference on Artificial Intelligence and Statistics},
  pages = 	 {4276--4284},
  year = 	 {2025},
  editor = 	 {Li, Yingzhen and Mandt, Stephan and Agrawal, Shipra and Khan, Emtiyaz},
  volume = 	 {258},
  series = 	 {Proceedings of Machine Learning Research},
  month = 	 {03--05 May},
  publisher =    {PMLR},
  url = 	 {https://proceedings.mlr.press/v258/mazaheri25a.html}
}

@article{de2020two,
  title={Two-way fixed effects estimators with heterogeneous treatment effects},
  author={De Chaisemartin, Cl{\'e}ment and d’Haultfoeuille, Xavier},
  journal={American economic review},
  volume={110},
  number={9},
  pages={2964--2996},
  year={2020},
  publisher={American Economic Association 2014 Broadway, Suite 305, Nashville, TN 37203}
}

@article{wager2018estimation,
  title={Estimation and inference of heterogeneous treatment effects using random forests},
  author={Wager, Stefan and Athey, Susan},
  journal={Journal of the American Statistical Association},
  volume={113},
  number={523},
  pages={1228--1242},
  year={2018},
  publisher={Taylor \& Francis}
}

@article{rosenbaum1983central,
  title={The central role of the propensity score in observational studies for causal effects},
  author={Rosenbaum, Paul R and Rubin, Donald B},
  journal={Biometrika},
  volume={70},
  number={1},
  pages={41--55},
  year={1983},
  publisher={Oxford University Press}
}

@book{peters2017elements,
  title={Elements of causal inference: foundations and learning algorithms},
  author={Peters, Jonas and Janzing, Dominik and Sch{\"o}lkopf, Bernhard},
  year={2017},
  publisher={The MIT Press}
}

@book{pearl2009causality,
  title={Causality},
  author={Pearl, Judea},
  year={2009},
  publisher={Cambridge university press}
}

@article{singh2019kernel,
  title={Kernel instrumental variable regression},
  author={Singh, Rahul and Sahani, Maneesh and Gretton, Arthur},
  journal={Advances in Neural Information Processing Systems},
  volume={32},
  year={2019}
}

@article{reifeis2022variance,
  title={On variance of the treatment effect in the treated when estimated by inverse probability weighting},
  author={Reifeis, Sarah A and Hudgens, Michael G},
  journal={American Journal of Epidemiology},
  volume={191},
  number={6},
  pages={1092--1097},
  year={2022},
  publisher={Oxford University Press}
}

@article{tran2023robust,
  title={Robust variance estimation and inference for causal effect estimation},
  author={Tran, Linh and Petersen, Maya and Schwab, Joshua and van der Laan, Mark J},
  journal={Journal of Causal Inference},
  volume={11},
  number={1},
  pages={20210067},
  year={2023},
  publisher={De Gruyter}
}

@article{matsouaka2023variance,
  title={Variance estimation for the average treatment effects on the treated and on the controls},
  author={Matsouaka, Roland A and Liu, Yi and Zhou, Yunji},
  journal={Statistical Methods in Medical Research},
  volume={32},
  number={2},
  pages={389--403},
  year={2023},
  publisher={SAGE Publications Sage UK: London, England}
}

@article{SriGreFukLanetal10,
  Author =	 {B. Sriperumbudur and A. Gretton and K. Fukumizu and
                  G. Lanckriet and B. Sch{\"o}lkopf},
  Journal =	 {Journal of Machine Learning Research},
  Pages =	 {1517-1561},
  Title =	 {Hilbert Space Embeddings and Metrics on Probability
                  Measures},
  Volume =	 11,
  Year =	 2010
}

@article{Rubin2005,
 author = {Donald B. Rubin},
 journal = {Journal of the American Statistical Association},
 number = {469},
 pages = {322--331},
 title = {Causal Inference Using Potential Outcomes: Design, Modeling, Decisions},
 volume = {100},
 year = {2005}
}

@article{Fischer2020Sobolev,
  author  = {Simon Fischer and Ingo Steinwart},
  title   = {Sobolev Norm Learning Rates for Regularized Least-Squares Algorithms},
  journal = {Journal of Machine Learning Research},
  year    = {2020},
  volume  = {21},
  number  = {205},
  pages   = {1--38},
}

@article{Vito2005,
author = {Vito, E. and Caponnetto, A. and Rosasco, L.},
title = {Model Selection for Regularized Least-Squares Algorithm in Learning Theory},
year = {2005},
volume = {5},
number = {1},
journal = {Foundation of Computational Mathematics},
}

@article{Smale2007LearningTE,
  title={Learning Theory Estimates via Integral Operators and Their Approximations},
  author={Stephen Smale and Ding-Xuan Zhou},
  journal={Constructive Approximation},
  year={2007},
  volume={26},
}

@inproceedings{Song2009Hilbert,
author = {Song, Le and Huang, Jonathan and Smola, Alex and Fukumizu, Kenji},
title = {Hilbert Space Embeddings of Conditional Distributions with Applications to Dynamical Systems},
year = {2009},
booktitle = {Proceedings of the 26th International Conference on Machine Learning},
pages = {961–968}
}

@inproceedings{Gre2012CME,
author = {Gr\"{u}new\"{a}lder, Steffen and Lever, Guy and Baldassarre, Luca and Patterson, Sam and Gretton, Arthur and Pontil, Massimilano},
title = {Conditional Mean Embeddings as Regressors},
year = {2012},
booktitle = {Proceedings of the 29th International Coference on International Conference on Machine Learning},
pages = {1803–1810}
}

@inproceedings{Park2020CME,
 author = {Park, Junhyung and Muandet, Krikamol},
 booktitle = {Advances in Neural Information Processing Systems},
 pages = {21247--21259},
 title = {A Measure-Theoretic Approach to Kernel Conditional Mean Embeddings},
 volume = {33},
 year = {2020}
}

@inproceedings{Li2022optimal,
title={Optimal Rates for Regularized Conditional Mean Embedding Learning},
author={Zhu Li and Dimitri Meunier and Mattes Mollenhauer and Arthur Gretton},
booktitle={Advances in Neural Information Processing Systems},
year={2022}
}

@article{Stuart2010Matching,
author = {Stuart, Elizabeth},
year = {2010},
pages = {1-21},
title = {Matching Methods for Causal Inference: A Review and a Look Forward},
volume = {25},
journal = {Statistical science : a review journal of the Institute of Mathematical Statistics},
}

@article{Rubin1996Matching,
 author = {Donald B. Rubin and Neal Thomas},
 journal = {Biometrics},
 number = {1},
 pages = {249--264},
 title = {Matching Using Estimated Propensity Scores: Relating Theory to Practice},
 volume = {52},
 year = {1996}
}

@inproceedings{Shi2019,
 author = {Shi, Claudia and Blei, David and Veitch, Victor},
 booktitle = {Advances in Neural Information Processing Systems},
 title = {Adapting Neural Networks for the Estimation of Treatment Effects},
 volume = {32},
 year = {2019}
}

@article{Athey2019Causal,
author = {Susan Athey and Julie Tibshirani and Stefan Wager},
title = {{Generalized random forests}},
volume = {47},
journal = {The Annals of Statistics},
number = {2},
pages = {1148 -- 1178},
year = {2019}
}

@inproceedings{chernozhukov2022riesznet,
author = {Chernozhukov, Victor and Newey, Whitney K. and Quintas-Martinez, Victor and Syrgkanis, Vasilis},
title = {RieszNet and ForestRiesz: Automatic Debiased Machine Learning with Neural Nets and Random Forests},
booktitle = {Proceedings of the  The 39th International Conference on Machine Learning},
year = {2022},
}

@MISC{IHDP,
  title     = "Infant Health and Development Program ({IHDP)}: Enhancing the
               Outcomes of Low Birth Weight, Premature Infants in the United
               States, 1985-1988",
  author    = "Gross, R. T.",
  year      =  1993
}

@inproceedings{Kawakami2025Moments,
author = {Kawakami, Yuta and Tian, Jin},
title = {Moments of causal effects},
year = {2025},
publisher = {JMLR.org},
booktitle = {Proceedings of the Forty-First Conference on Uncertainty in Artificial Intelligence},
articleno = {89},
numpages = {33},
location = {Rio de Janeiro, Brazil},
series = {UAI '25}
}

@article{Kennedy2023Semi,
    author = {Kennedy, E H and Balakrishnan, S and Wasserman, L A},
    title = {Semiparametric counterfactual density estimation},
    journal = {Biometrika},
    volume = {110},
    number = {4},
    pages = {875-896},
    year = {2023},
    month = {03},
    issn = {1464-3510},
}

@article{Krikamol2021Counterfactual,
author = {Muandet, Krikamol and Kanagawa, Motonobu and Saengkyongam, Sorawit and Marukatat, Sanparith},
title = {Counterfactual mean embeddings},
year = {2021},
publisher = {JMLR.org},
volume = {22},
number = {1},
journal = {Journal of Machine Learning Research},
articleno = {162},
numpages = {71},
}

@article{Singh2023Kernel,
    author = {Singh, R and Xu, L and Gretton, A},
    title = {Kernel methods for causal functions: dose, heterogeneous and incremental response curves},
    journal = {Biometrika},
    volume = {111},
    number = {2},
    pages = {497-516},
    year = {2023},
    month = {07},
}
\bibliographystyle{abbrvnat}

\newpage
\appendix
\onecolumn

\section{Empirical solution for CVTE conditioned on covariates} \label{sec:emp-cvte-identity}

In this section, we discuss the CVTE solution where $X$ contains the conditioning $V$. In such a case, we cannot use the empirical solution presented in \cref{thm:cvte-emp-sol}, since the identity operator is not the Hilbert-Schmidt operator when the RKHS is of infinite dimension, which contradicts the assumptions in \cref{prop:cme}. Instead, we use kernel tensor feature similar to kernel CATE estimator proposed in \citep{singh2019kernel}.

Given data $\{x_i, y_i, a_i\}$, where $x_i = (\tilde x_i, v_i)$, define composite kernel function $k$ as 
\begin{align*}
    k(x_i, x_j) = k_{\tilde X}(\tilde x_i, \tilde x_j) k_V(v_i, v_j),
\end{align*}
where $k_{\tilde X}, k_V$ are arbitrary kernel functions. The solution of CVTE can be given as 
\begin{align*}
    \widehat{\mathrm{CVTE}}(v) &= \sum_{i=1}^n w_i(v)\left(\hat{g}_1(\tilde x_i, v) +  \hat{g}_0(\tilde x_i, v) - 2 \hat{f}_1(\tilde x_i, v)\hat{f}_0(\tilde x_i, v)\right) - \left(\sum_{i=1}^n  w_i(v) \left(\hat{f}_1(\tilde x_i, v) - \hat{f}_0(\tilde x_i, v)\right)\right)^2,
\end{align*}
where $f_a, g_a$ are the kernel ridge regressors defined in \eqref{eq:f-krr}, \eqref{eq:g-krr} using the composite kernel $k$ and weight $w_i(v)$ defined in \cref{thm:cvte-emp-sol}. This idea uses the fact that for any function $h(x)$
\begin{align*}
    \expect{h(X)|V=v} = \expect[\tilde X]{h(\tilde X, v)\middle|V=v},
\end{align*}
where right hand side can be computed by the conditional mean embedding.

\section{Theoretical Results}

In this section, we provide deferred proofs from the main paper.

\subsection{Proof of \cref{thm:nonidentifiability}} \label{sec:proof-non-identifiability}

Let $X\sim \mathcal{N}(0,1)$ and $A$ is sampled from Bernoulli distribution with $\prob{A=1}=0.5$. Consider two distributions for potential outcomes.

\begin{align*}
    \text{Case 1:}& \quad Y^{(0)} \sim \mathcal{N}(0, 1), \quad Y^{(1)} = Y^{(0)} \\ \text{Case 2:}& \quad Y^{(0)} \sim \mathcal{N}(0, 1), \quad Y^{(1)} = -Y^{(0)}
\end{align*}

Since $Y^{(0)}, Y^{(1)} \indepe A$, we can see that $X$ satisfies the ignorability \cref{assum:ignorability}. Although these two has different VTEs, the observational distribution $P(Y|A, X)$ is given as $\mathcal{N}(0,1)$. This shows the  existence of the data generation processes stated in the theorem.

\subsection{Proof of \cref{thm:vte-identification,thm:cvte-identification}} \label{sec:proof-identifiability}

From the definition of VTE, we have
\begin{align*}
    \mathrm{VTE} &= \Var{Y^{(1)} - Y^{(0)}}\\
    &= \expect{(Y^{(1)} - Y^{(0)})^2} -\expect{(Y^{(1)} - Y^{(0)})}^2\\
    &= \expect{(Y^{(1)})^2} -2\expect{Y^{(1)}Y^{(0)}} +  \expect{(Y^{(0)})^2} -\left(\expect{Y^{(1)}} - \expect{Y^{(0)}}\right)^2
\end{align*}
For $a \in \{0,1\}$ and $n \in \{1,2\}$, we have 
\begin{align*}
    \expect{(Y^{(a)})^n} &= \expect[X]{\expect{(Y^{(a)})^n \middle| X}}\\
    &= \expect[X]{\expect{(Y^{(a)})^n \middle| A=a, X}} \quad \because \text{\cref{assum:ignorability}}\\
    &= \expect[X]{\expect{Y^n \middle| A=a, X}} \quad \because \text{No inference}\\
    &=\begin{cases}
        \expect[X]{f_a(X)} & (n=1)\\
        \expect[X]{g_a(X)} & (n=2)
    \end{cases}
\end{align*}
Furthermore, 
\begin{align*}
    \expect{Y^{(1)}Y^{(0)}} &= \expect[X]{\expect{Y^{(1)}Y^{(0)} \middle| X}}\\
    &= \expect[X]{\expect{Y^{(1)}\middle|X}\expect{Y^{(0)} \middle| X}} \quad\because\text{\cref{assum:discomposional}}\\
    &=  \expect[X]{f_0(X)f_1(X)}
\end{align*}
Rearranging these results gives what we wanted. We can show the identification results for CVTE using a similar discussion as well.  If we would like to relax \cref{assum:discomposional}, then we observe that the second step in the final block, which applies \cref{assum:discomposional}, may also add a correlation term between $Y^{(1)}$ and $Y^{(0)}$. Since this term may not be observable, its value must be assumed, or the results would need to be relaxed into upper and lower bounds.

\subsection{Proofs in VTE estimation} \label{sec:proof-vte-estimation}

We first state the detailed convergence results for kernel ridge regression and kernel mean embeddings.

\begin{prop}[{Corollary of \citealp[Theorem 1]{Fischer2020Sobolev}}]  \label{prop:detail-krr}
    Assume that kernel is bounded $k(x,x) \leq \kappa^2$ and $f_a, g_a$ are in the $\beta$-power space $\calH^\beta_\calX$ with $\beta \in (1,2)$. We assume the eigen-value decay satisfies a polynomial upper bound of order $1/p$ for $p > 0$. If $|Y - f_a(X)|$ and $|Y^2 - g_a(X)|$ are upper-bounded by a constant $M$ almost surely, we have
    \begin{align*}
        \|\hat{f}_a - f_a\|_{\calH_\calX}  \leq \log(16/\delta) \mathfrak{C}_{f_a} n_a^{-\frac{\beta}{2(\beta + p)}}, \quad \|\hat{g}_a - g_a\|_{\calH_\calX} \leq \log(16/\delta) \mathfrak{C}_{g_a} n_a^{-\frac{\beta}{2(\beta + p)}}
    \end{align*}
    for all $a \in \{0,1\}$ with the probabilty of $1-\delta$ by setting $\lambda_{f_a},\lambda_{g_a} = O(n^{-1/(\beta+p)})$, where $\mathfrak{C}_{f_a}, \mathfrak{C}_{g_a}$ denotes the constants does not depends on $n_a, \delta$.
\end{prop}

We also use the following lemma for applying Hoeffding's inequality.
\begin{prop}[{Corollary of \citealp[Lemma 14]{Fischer2020Sobolev}}] 
    We have 
    \begin{align*}
        \|\hat{f}_a\|_{\calH_\calX} \leq \|f_a\|_{\calH_\calX}, \quad \|\hat{g}_a\|_{\calH_\calX} \leq \|g_a\|_{\calH_\calX}
    \end{align*}
\end{prop}

Now, we prove \cref{thm:vte-emp-sol}. Under the condition of \cref{prop:detail-krr}, for all $i \in [n]$, we have 
\begin{align*}
    \left|\hat{g}_a(x_i) - g_a(x_i)\right| &\leq \kappa \|\hat{g}_a - g_a\|_{\calH_\calX}   \leq \kappa \log(16/\delta) \mathfrak{C} n_a^{-\frac{\beta}{2(\beta + p)}}\\
    \left|\hat{f}_0(x_i)\hat{f}_1(x_i) - f_0(x_i)f_1(x_i)  \right| &\leq \left|\hat{f}_0(x_i)\hat{f}_1(x_i) - \hat{f}_0(x_i)f_1(x_i)\right| + \left|\hat{f}_0(x_i)f_1(x_i) - f_0(x_i)f_1(x_i)  \right|\\
    &= |\hat{f}_0(x_i)||\hat{f}_1(x_i) - f_1(x_i)| + |f_1(x_i)||\hat{f}_0(x_i) - f_0(x_i)|\\
    &\leq \kappa(\|f_1\|_{\calH_\calX} + \|f_0\|_{\calH_\calX}) \log(16/\delta) \mathfrak{C} (n_0^{-\frac{\beta}{2(\beta + p)}}+n_1^{-\frac{\beta}{2(\beta + p)}})\\
    \left|\hat{f}_a(x_i) - f_a(x_i)\right| &\leq \kappa \|\hat{f}_a - f_a\|_{\calH_\calX}   \leq \kappa \log(16/\delta) \mathfrak{C} n_a^{-\frac{\beta}{2(\beta + p)}}
\end{align*}
for 
\begin{align*}
    \mathfrak{C} = \max(\mathfrak{C}_{f_1}, \mathfrak{C}_{f_0}, \mathfrak{C}_{g_1}, \mathfrak{C}_{g_0}).
\end{align*}
Therefore, for $\widetilde{\mathrm{VTE}}$ defined
\begin{align*}
       \widetilde{\mathrm{VTE}} &= \frac1n \sum_{i=1}^n g_1(x_i) +  g_0(x_i) - 2 f_1(x_i)f_0(x_i) - \left(\frac1n \sum_{i=1}^n f_1(x_i) - f_0(x_i)\right)^2,
   \end{align*}
we have,
\begin{align*}
    &|\widetilde{\mathrm{VTE}} - \widehat{\mathrm{VTE}}| \\
    &\leq \kappa(1 + 2 \|f_1\|_{\calH_\calX} + 2\|f_0\|_{\calH_\calX})\log(16/\delta) \mathfrak{C}(n_0^{-\frac{\beta}{2(\beta + p)}} + n_1^{-\frac{\beta}{2(\beta + p)}}) \\
    &\quad + \left|\left(\frac1n \sum_{i=1}^n (f_1(x_i)-\hat{f}_1(x_i)) - (f_0(x_i)-\hat{f}_0(x_i))\right)\left(\frac1n \sum_{i=1}^n (f_1(x_i)+\hat{f}_1(x_i)) - (f_0(x_i)+\hat{f}_0(x_i))\right)\right|\\
    &\leq \kappa(1 + 2 \|f_1\|_{\calH_\calX} + 2\|f_0\|_{\calH_\calX})\log(16/\delta) \mathfrak{C} n_a^{-\frac{\beta}{2(\beta + p)}} +  (2 \|f_1\|_{\calH_\calX} + 2\|f_0\|_{\calH_\calX}) (2\kappa \log(16/\delta) \mathfrak{C} (n_0^{-\frac{\beta}{2(\beta + p)}} + n_1^{-\frac{\beta}{2(\beta + p)}})),
\end{align*}
with probability $1-4\delta$. From Hoeffding's inequality we have 
\begin{align*}
    &\left|\sum_{i=1}^n g_1(x_i) +  g_0(x_i) - 2 f_1(x_i)f_0(x_i) - \expect{ g_1(X) +  g_0(X) - 2 f_1(X)f_0(X)} \right| \leq O_p(1/n^{-1/2}),\\
    &\left|\frac1n \sum_{i=1}^n f_1(x_i) - f_0(x_i) - \expect{f_1(X) - f_0(X)}\right| \leq O_p(1/n^{-1/2}).
\end{align*}
Therefore, we have
\begin{align*}
    |\mathrm{VTE} - \widehat{\mathrm{VTE}}| &\leq |\mathrm{VTE} - \widetilde{\mathrm{VTE}}| + |\widetilde{\mathrm{VTE}} - \widehat{\mathrm{VTE}}|\\
    &\leq O_p(n_0^{-\frac{\beta}{2(\beta + p)}} + n_1^{-\frac{\beta}{2(\beta + p)}}).
\end{align*}
From this, we can see that the convergence rate is equal to the kernel ridge regression rate in \cref{prop:detail-krr}.

\subsection{Proofs in CVTE Estimation}

We first prove \cref{lem:cme-to-cvte}.
\begin{proof}
    From \cref{thm:cvte-identification}, under \cref{assum:discomposional,assum:ignorability},
    \begin{align*}
        \mathrm{CVTE}(v) = \expect[X]{g_1(X) + g_0(X) -2f_1(X)f_0(X)|V=v}  - \expect[X]{f_1(X) - f_0(X)|V=v}^2.
    \end{align*}
    Now if $f_a, g_a \in \mathcal{H}_\calX$, from linearlity, we have
    \begin{align*}
        \expect[X]{g_1(X) + g_0(X)|V=v} &= \expect[X]{\braket[\calH_\calX]{g_1 + g_0, \phi(X)}|V=v}\\
        &= \braket[\calH_\calX]{g_1 + g_0, \expect[X]{\phi(X)}|V=v}\\
        &= \braket[\calH_\calX]{g_1 + g_0, \mu_{X|V}(v)}\\
        \expect[X]{f_1(X)f_0(X)|V=v} &= \expect[X]{\braket[\calH_\calX]{f_1, \phi(X) \otimes \phi(X) f_0}|V=v}\\
        &= \braket[\calH_\calX]{f_1, \expect[X]{\phi(X) \otimes \phi(X)|V=v} f_0}\\
        &= \braket[\calH_\calX]{f_1, C_{XX|V}(v) f_0}\\
        \expect[X]{f_1(X) - f_0(X)|V=v} &= \expect[X]{\braket[\calH_\calX]{f_1 - f_0, \phi(X)}|V=v}\\
        &= \braket[\calH_\calX]{f_1 - f_0, \expect[X]{\phi(X)}|V=v}\\
        &= \braket[\calH_\calX]{f_1 - f_0, \mu_{X|V}(v)}
    \end{align*}
    Combine them yields what we wanted.
\end{proof}

Next, we state the convergence result of conditional mean embedding $\mu_{X|V}$.

\begin{prop}[{Corollary of \citealp[Thoerem 1]{singh2019kernel}}] 
    Assume $k(x,x) \leq \kappa^2$ and $k_V(v,v) \leq Q^2$ and the operator operator $H_\rho: f\in\calH_\calX \to \expect{f(X)|V=\cdot}\in\calH_\calV$ is  Hilbert-Schmidt. Furthermore, we assume there exists Hilbert-Schmidt operator $G$ and constant $c \in (1,2]$ such that 
    \begin{align*}
        H_\rho = \expect{\psi(V) \otimes \psi(V)}^{\frac{c-1}{2}} G,
    \end{align*}
    where $\psi(v) = k_V(v,\cdot)$. Then, with probability $1-\delta$, we have
    \begin{align*}
        \left\| \mu_{X|V}(v) -  \sum_{i=1}^n w_i(v) \phi(x_i)\right\|_{\mathcal{H}_\calX} \leq \mathfrak{D} (n\log(1/\delta))^{-\frac{c-1}{2(c+1)}}
    \end{align*}
    for $\lambda = O(n^{-\frac{1}{c+1}})$, where $\mathfrak{D}$ is a constant that is independent of data size.
\end{prop}
We can prove a similar convergence result for $C_{XX|V}$ as well.
Given this, we can prove \cref{thm:cvte-emp-sol} as follows.
\begin{proof}
Let $\hat{\mu}_{X|V}(v), \hat{C}_{XX|V}$ defined as 
\begin{align*}
    \hat{\mu}_{X|V}(v) =  \sum_{i=1}^n w_i(v) \phi(x_i), \quad \hat{C}_{XX|V}(v) =  \sum_{i=1}^n w_i(v) \phi(x_i) \otimes \phi(x_i).
\end{align*}
Then, from reproducing characteristics, we have 
\begin{align*}
    \widehat{\mathrm{CVTE}}(v)  = \braket[\calH_\calX]{\hat{g}_1 + \hat{g}_0, \hat{\mu}_{X|V}(v)} - 2\braket[\calH_\calX]{\hat{f}_1, \hat{C}_{XX|V}(v)\hat{f}_0} - \braket[\calH_\calX]{\hat{f}_1-\hat{f}_0, \hat{\mu}_{X|V}(v)}^2.
\end{align*}
Therefore, from \cref{lem:cme-to-cvte}, we have
\begin{align*}
    |\widehat{\mathrm{CVTE}}(v) - \mathrm{CVTE}(v)| &\leq \left|\braket[\calH_\calX]{\hat{g}_1 + \hat{g}_0, \hat{\mu}_{X|V}(v)} - \braket[\calH_\calX]{{g}_1 + {g}_0, {\mu}_{X|V}(v)}\right| \\
    &\qquad +2 \left|\braket[\calH_\calX]{\hat{f}_1, \hat{C}_{XX|V}(v)\hat{f}_0} - \braket[\calH_\calX]{{f}_1, {C}_{XX|V}(v){f}_0}\right| \\
    &\qquad +\left|\braket[\calH_\calX]{\hat{f}_1-\hat{f}_0, \hat{\mu}_{X|V}(v)}^2 - \braket[\calH_\calX]{{f}_1-{f}_0, {\mu}_{X|V}(v)}^2\right|.
\end{align*}
Now, for each term,
\begin{align*}
    \left|\braket[\calH_\calX]{\hat{g}_1 + \hat{g}_0, \hat{\mu}_{X|V}(v)} - \braket[\calH_\calX]{{g}_1 + {g}_0, {\mu}_{X|V}(v)}\right|  &\leq \|{\mu}_{X|V}(v)\|_{\calH_\calX}(\|\hat{g}_1 - g_1\|_{\calH_\calX} + \|\hat{g}_0 - g_0\|_{\calH_\calX})\\
    &\quad +\|\hat{g}_1 + \hat{g}_0\|_{\calH_\calX} \|\hat{\mu}_{X|V}(v) - {\mu}_{X|V}(v) \|_{\calH_\calX},\\
    \left|\braket[\calH_\calX]{\hat{f}_1, \hat{C}_{XX|V}(v)\hat{f}_0} - \braket[\calH_\calX]{{f}_1, {C}_{XX|V}(v){f}_0}\right| &\leq
    \|\hat{f}_1-f_1\|_{\calH_\calX} \|\hat{C}_{XX|V}(v)\hat{f}_0\|_{\calH_\calX}\\
    &\quad+ \|f_1\|_{\calH_\calX}\|{C}_{XX|V}(v) - \hat{C}_{XX|V}(v)\|_{\mathrm{op}}\|\hat{f}_0\|_{\calH_\calX}\\
    &\quad\quad+\|f_1\|_{\calH_\calX}\|{C}_{XX|V}(v)\|_{\mathrm{op}}\|f_0-\hat{f}_0\|_{\calH_\calX},\\
    \left|\braket[\calH_\calX]{\hat{f}_1-\hat{f}_0, \hat{\mu}_{X|V}(v)}^2 - \braket[\calH_\calX]{{f}_1-{f}_0, {\mu}_{X|V}(v)}^2\right| &\leq \left|\braket[\calH_\calX]{\hat{f}_1-\hat{f}_0, \hat{\mu}_{X|V}(v)} + \braket[\calH_\calX]{{f}_1-{f}_0, {\mu}_{X|V}(v)}\right|\\
    &\quad \times \left|\braket[\calH_\calX]{\hat{f}_1-\hat{f}_0, \hat{\mu}_{X|V}(v)} - \braket[\calH_\calX]{{f}_1-{f}_0, {\mu}_{X|V}(v)}\right|\\
    &\leq \left|\|\hat{f}_1-\hat{f}_0\|_{\calH_\calX} \|\hat{\mu}_{X|V}(v)\|_{\calH_\calX} + \|{f}_1-{f}_0\|_{\calH_\calX}\|{\mu}_{X|V}(v)\|_{\calH_\calX}\right|\\
    &\quad \times \left| \|{\mu}_{X|V}(v)\|_{\calH_\calX}(\|\hat{f}_1 - f_1\|_{\calH_\calX} + \|\hat{f}_0 - f_0\|_{\calH_\calX})\right.\\
    &\quad\quad \left.+\|\hat{f}_1 - \hat{f}_0\|_{\calH_\calX} \|\hat{\mu}_{X|V}(v) - {\mu}_{X|V}(v) \|_{\calH_\calX}\right|
\end{align*}

From these, we see $|\widehat{\mathrm{CVTE}}(v) - \mathrm{CVTE}(v)| \to 0$ if 
\begin{align*}
    \|\hat{f}_a - f_a\|_{\calH_\calX} \to 0, \|\hat{g}_a - g_a\|_{\calH_\calX} \to 0, \|\hat{\mu}_{X|V}(v) - {\mu}_{X|V}(v) \|_{\calH_\calX}\to 0
\end{align*}
for $a \in \{0,1\}$. The rate is dominated by the slowest convergence rate above.
\end{proof}

\end{document}